\documentclass{article}
\usepackage{tkai}
\usepackage{float} 
\usepackage[T1]{fontenc}
\usepackage[utf8]{inputenc}
\usepackage[T1]{fontenc}
\usepackage[utf8]{inputenc}
\usepackage{nicefrac}       
\usepackage{microtype}      

\usepackage{amsfonts}       
\usepackage{amsmath}
\usepackage{amssymb}
\usepackage{mathtools}
\usepackage[toc,page]{appendix}
\usepackage{enumitem}
\usepackage{graphics}
\usepackage{graphicx}
\usepackage{xcolor}
\usepackage{subcaption}

\usepackage[export]{adjustbox}
\usepackage{algorithm}
\usepackage{algorithmicx}
\usepackage[noend]{algpseudocode}
\usepackage{stmaryrd}
\newtheorem{theorem}{Theorem}[section]
\newtheorem{definition}[theorem]{Definition}

\newtheorem{corollary}[theorem]{Corollary}

\newtheorem{lemma}[theorem]{Lemma}

\newenvironment{proof}{\paragraph{Proof:}}{\hfill$\square$}
\newenvironment{sproof}{\paragraph{Proof Sketch:}}{\hfill$\square$}

\newcommand{\chara}{\mathcal{\mathfrak{a}}}

\usepackage{newfloat}
\usepackage{listings}
\DeclareCaptionStyle{ruled}{labelfont=normalfont,labelsep=colon,strut=off} 
\floatstyle{ruled}
\newfloat{listing}{tb}{lst}{}
\floatname{listing}{Listing}

\title{On the Computational Complexity and Formal Hierarchy of \\Second Order Recurrent Neural Networks}

\author{%
  Ankur Mali \\
  University of South Florida\\
  \texttt{ankurarjunmali@usf.edu}
  \And
  Alexander Ororbia \\
  Rochester Institute of Technology \\
  \texttt{ago@cs.rit.edu}
  \AND
  Daniel Kifer \\
  Pennsylvania State University \\
  \texttt{duk17@psu.edu}
  \And
  Lee Giles \\
  Pennsylvania State University \\
  \texttt{clg20@psu.edu}  
}

\begin{document}

\maketitle

\begin{abstract}
Artificial neural networks (ANNs) with recurrence and self-attention have been shown to be Turing-complete (TC). However, existing work has shown that these ANNs require multiple turns or unbounded computation time, even with unbounded precision in weights, in order to recognize TC grammars. However, under constraints such as fixed or bounded precision neurons and time, ANNs without memory are shown to struggle to recognize even context-free languages. In this work, we extend the theoretical foundation for the $2^{nd}$-order recurrent network ($2^{nd}$ RNN) and prove there exists a class of a $2^{nd}$ RNN that is Turing-complete with bounded time. This model is capable of directly encoding a transition table into its recurrent weights, enabling bounded time computation and is interpretable by design. We also demonstrate that $2$nd order RNNs, without memory, under bounded weights and time constraints, outperform modern-day models such as vanilla RNNs and gated recurrent units in recognizing regular grammars. We provide an upper bound and a stability analysis on the maximum number of neurons required by $2$nd order RNNs to recognize any class of regular grammar. Extensive experiments on the Tomita grammars support our findings, demonstrating the importance of tensor connections in crafting computationally efficient RNNs. Finally, we show $2^{nd}$ order RNNs are also interpretable by extraction and can extract state machines with higher success rates as compared to first-order RNNs. Our results extend the theoretical foundations of RNNs and offer promising avenues for future explainable AI research.
\end{abstract}

\section{Introduction}
\label{sec:intro}

Artificial neural networks (ANNs) have achieved impressive results across a wide variety of natural language processing (NLP) tasks. Two of the most promising approaches in NLP are those based on attention-based mechanisms and recurrence. One key way to evaluate the computational capability of these models is by comparing them with the Turing machine (TM). Importantly, it has been shown that recurrent neural networks (RNNs)\cite{siegelmann94} and transformers \cite{attn_turing} are Turing complete when leveraging unbounded precision and weights. However, these classes of ANNs, by design, lack structural memory and thus empirically fail to learn the structure of data when tested on complex algorithmic patterns. Researchers have shown that, under restrictions or constraints such as bounded precision, RNNs operate more closely to finite automata \cite{Merrill19}, and transformers notably fail to recognize grammars \cite{Hahn_2020} produced by pushdown automata. 

Another complementary research direction is focused on coupling ANNs with memory structures \cite{joulin2015inferring, mali2021recognizing2,mali2020neural, grefenstette2015learning, graves2014neural, graves2016hybrid} such as stacks, queues, and even tapes in order to overcome memory issues. It has been shown that the neural Turing machine (NTM) \cite{graves2014neural} and the differentiable neural computer (DNC) \cite{graves2016hybrid}, are not Turing complete in their current form and can only model space-bounded TMs. However, recent work has theoretically derived a computational bound for RNNs either coupled to a growing memory module \cite{chung2021turing} or augmented with a stack-like structure \cite{mali2021recognizing, nndpa1998sun, grefenstette2015learning, joulin2015inferring, stogin2020provably} -- these models have been shown to be Turing complete (TC) with bounded precision neurons. One of the earliest works in analyzing RNN computational capabilities was that of Siegelman and Sontag \cite{siegelmann95}, which theoretically demonstrated that an RNN with unbounded precision was TC if the model satisfied two key conditions:
\begin{enumerate}[noitemsep,nolistsep]
    \item it was capable of computing or extracting dense representations of the data, and
    \item it implemented a mechanism to retrieve these stored (dense) representations. 
\end{enumerate} 
Notably, this work demonstrated that simple RNNs could model any class of Turing complete grammars and, as a result, were universal models. These proofs have recently been extended to other kinds of ANNs without recurrence \cite{perez2019turing}. In particular, transformers with residual connections as well as the Neural GPU equipped with a gating mechanism were shown to be Turing complete \cite{perez2019turing}. Despite these important efforts, there is far less work on theoretically deriving the computational limits of RNN with finite precision and time. Therefore, in this paper, we specifically consider these constraints and formally construct the smallest TC RNNs possible as well as show that we only require $11$ unbounded precision neurons to simulate any TM. We also prove that a class of RNNs can simulate a TM in finite-time $O(T)$, where $T$ is simulation time. In other words a system is finite if there is formal guarantee that simulation of TM will be completed by any system with available resources. If $T>1$, then computational becomes infinite, and there is no guarantee that we fill find the answer with curretnly available compute.

To achieve this, we simulate Turing machines with RNNs built from tensor synaptic connections (TRNNs). It has been shown that tensor connections are interpretable and offer practical as well as computational benefits \cite{mali2019neural, omlin1996constructing}. Furthermore, when augmented with differentiable memory, such networks have been shown to be \emph{Turing equivalent with finite precision and time} \cite{stogin2020provably}. However, none of the prior memory-less RNNs have been shown to be \emph{Turing complete with finite time}. To this end, we will formally prove that: 1) an unbounded precision TRNN requires a smaller number of neurons and is equivalent to a universal Turing machine, and 2) a bounded precision TRNN can model any class of regular grammars and is more powerful than RNNs and transformers in recognizing these languages. As a result, the TRNN offers the following benefits: 
1) \textbf{Computational benefit:} The number of hidden states required to construct a TM via a TRNN is comparatively less than other ANNs,  
2) \textbf{Rule insertion and interpretability:} One can add/program the transition rules of a Turing complete grammar directly into weights of TRNN, making the model interpretable by design, and 
3) \textbf{Explainability:} One can extract  the rules as well as the underlying state machine from a well-trained TRNN, making the model explainable via extraction.

This work makes the following contributions:
\begin{itemize}
    \item We formally prove the Turing completeness of a TRNN with unbounded precision. 
    \item We prove that one can simulate any TM in real-time and that the TRNN, by design, is interpretable since the number of states in the TRNN is equivalent to a TM.
    \item We prove that a bounded precision TRNN requires $n+1$ neurons in order to recognize any regular grammar, where $n$ is the total number of states in a deterministic finite automaton (DFA). 
\end{itemize}
The remainder of the work is organized as follows. Section \ref{sec:notation} describes the background and notation, presenting the Turing machine (TM) definition. Section \ref{sec:trnn_turing_complete} defines the TRNN architecture and its equivalence with a TM. Section \ref{sec:trnn_bounded_prec} defines the TRNN architecture with bounded precision and weights and establishes an equivalence to the DFA. Section \ref{sec:experiments} presents experiments on the Tomita grammar datasets. Finally, we draw final conclusions. 

\section{Background and Notation}
\label{sec:notation}

In the theory of computation, Turing machines (TMs) sit at the top of Chomsky's hierarchy as they can accept Type-0 unrestricted grammars. Problems that TMs cannot solve are considered to be undecidable. Conversely, if a problem is decidable, then there exists a Turing machine that can solve it. If a system of rules formulated to solve a problem can be simulated using a Turing machine, then the system is said to be Turing-complete (TC). Thus, we can see that the concept of Turing-completeness is absolutely essential for advancing a fundamental understanding of deep neural networks.

A Turing machine is formally defined as tuple $M = \{Z, \Sigma, R, \delta, q_0, b, F\}$ where $Z$ is a finite and non-empty set of states, $\Sigma$ is the set of input alphabets, and $b$ is the blank symbol. $R$ is a non-empty set of tape alphabets such that $\Sigma \subseteq R \backslash \{b\}$. $q_0 \in Z$ is the starting state and $F \subseteq Z$ is the set of final or accepting states. The transition function $\delta$ is a partial function defined as $\delta : Z\backslash F \times R \rightarrow Z \times R \times \{-1, 0, 1 \}$.  
We will consider a TM with $m$ tape symbols and $n$ states that can be defined using a finite set as follows: 
\begin{multline}
    \delta = \{<r_i, z_j| r_u, z_v A> \\ | i, u = 1,2..m, j,v = 1,2..n\}
\end{multline}
where $r_i$ represents the read symbol extracted from the tape, $z_j$ represents the controller's current state, and $A$ represents the action of the tape. As noted, each tuple in $\delta$ represents transition rules or functions. For instance, while reading $r_i, z_j| r_u, z_v A$, the control head in state $z_j$, after reading symbol $r_i$ from the tape, will perform one of three actions: 
\textbf{1)} delete the current value at $r_i$ from the tape and rewrite the current tape symbol as $r_u$, 
\textbf{2)} change its state from $z_j$ and move to the next state $z_u$, and 
\textbf{3)} move the head position according to the value of $A$; if $A=1$, then move one unit to the right, if $A=-1$, then move one unit to the left, and if $A=O$, then stay in the same state and do not move. We introduce an additional state known as the ``halt state'' to model illegal strings. Therefore, an input string is said to be accepted by Turing machine $M$ if $M$ reaches the final state or eventually halts and stays in that state forever. Hence, one can model this using the following transition rule or function: $<r_i,z_h |r_i,z_h A>$, where $A = 0$. 

A second ($2^{nd}$) order, or tensor, recurrent neural network (TRNN) is an RNN consisting of $n$ neurons; ideally, $n$ should be equivalent to the number of states in a TM grammar. The value of a neuron at time $t \in {1,2,3...n}$ is indexed using $i$, the state transition of the TRNN is computed by first taking an affine transformation followed by a nonlinear activation function as follows:
\begin{align}\label{eq:second_order}
    z_{i+1}^{t+1} = h_H(W_{ijk}^{t}z_jx_k + b_i)
\end{align}
where $z$ is the state transition of the TRNN such that $z_t \in \mathbb{Q}$ and $\mathbb{Q}$ is a set of rational numbers, $W \in \mathbb{R}^{n \times n \times n}$ contains the $2^{nd}$ order recurrent weight matrices, $b \in \mathbb{R}^n$ are the biases, $h_H$ is the nonlinear activation function such as the logistic sigmoid, and $x$ is the input symbol at time $t$.  
For simplicity, we will consider the saturated-linear function in this work, which is represented as follows:
\begin{align}
    h_H(z) = \sigma(z) &:=
    \left\{ 
    \begin{array}{rl}
      0 & \text{if}\quad z < 0 \\ 
      z & \text{if}\quad 0 \leq z \leq 1 \\ 
      1 & \text{if}\quad x > 1.
    \end{array}
    \right.\ \label{eqn:read}
\end{align}
Therefore, $z_t \in (\mathbb{Q} \cap [0,1])^n$ for all $t> 0$. In the appendix, we provide a notation table that contains all symbols used in this work alongside their definitions.

\section{Tensor Recurrent Neural Networks are \\Turing Complete}
\label{sec:trnn_turing_complete}

Later, we will show that, without loss of generality and inducing only a very small error $\epsilon$, one can obtain similar results with a sigmoid activation. To work with the sigmoid, we will restate the following lemma from \cite{stogin2020provably}:
\begin{lemma} \label{lemma:stack_op_lem_2}
  Let $\bar{Z}$ be a state tensor with components $\bar{Z}_i\in\{0,1\}$, and let $Z$ be a tensor satisfying:
  $$||Z-\bar{Z}||_{L^\infty} \le \epsilon_0$$
  for some $\epsilon_0< \frac12$.
  
  Then, for all sufficiently small $\epsilon > 0$, for all $H$ sufficiently large depending on $\epsilon_0$ and $\epsilon$,
  $$\max_i\left|\bar{Z}_i -h_H\left(Z_i-\frac12\right)\right| \le \epsilon.$$
  where $h_H(x)=\frac{1}{1+e^{-Hx}}$. 
\end{lemma}
Note that $h_H(0)=\frac12$ and $h_H(x)$ decreases to $0$ as $Hx\rightarrow -\infty$ and increases to $1$ as $Hx\rightarrow\infty$ and that scalar $H$ is a sensitivity parameter. In practice, we will often show that $x$ is bounded away from $0$ and then assume $H$ to be a positive constant sufficiently large such that $h_H(x)$ is as close as desired to either $0$ or $1$, depending on the sign of the input $x$. 
\begin{proof}
For simplicity, we will assume $\bar{Z}$ to be a tensor represented in vector form.

Since $\epsilon_0<\frac12$, choose $H$ sufficiently large such that:
\begin{align*}
    h_H\left(-\left(\frac12-\epsilon_0\right)\right) = 1-h_H\left(\frac12-\epsilon_0\right) \le \epsilon.  
\end{align*}
  

\noindent
Note that we have the following:
  \begin{align*}
    \left|\bar{Z}_i -h_H\left(Z_i-\frac12\right)\right| = \left|\bar{Z}_i -h_H\left(\bar{Z}_i-\frac12+(Z_i-\bar{Z}_i)\right)\right|.  
  \end{align*}
  Now pick an arbitrary index $i$. If at position $i$ we obtain $\bar{Z}_i=1$, then the following condition is obtained,
  \begin{align*}
     \left|\bar{Z}_i-h_H\left(Z_i-\frac12\right)\right| = 1-h_H\left(\frac12+(Z_i-\bar{Z}_i)\right) \le 1-h_H\left(\frac12-\epsilon_0\right) \le \epsilon. 
  \end{align*}
  \noindent
  However, if at position $i$ we have  $\bar{Z}_i=0$, then we obtain condition below: 
  \begin{align*}
      \left|\bar{Z}_i-h_H\left(Z_i-\frac12\right)\right| = h_H\left(-\frac12+(Z_i-\bar{Z}_i)\right) \le h_H\left(-\left(\frac12-\epsilon_0\right)\right) \le \epsilon 
  \end{align*}
  In both cases, we can see that vector $\bar{Z}$ is close to the ideal vector $Z$ which proves the lemma. 
\end{proof}

In order to simulate a Turing machine $M$ by a TRNN, we will construct a neural network that can encode the configuration of the TM with a tape using unbounded precision. We wish to create a binary mapping between the configuration of $M$ (with the transition function) and the neural activities of a TRNN, such that a one-to-one mapping is preserved. The neural network simulator ($S$) constructed using a TRNN contains the following components or tuples to model $M$:
\begin{align}
    S = ({n,z, W^z, h_H, z_0})
\end{align}
where transition table ($m \times n$) is indexed using rows ($n =$ total of neurons) and columns ($m =$ tape size), $Z$ is the set of all possible states (where ${z=z_{ij}, -n/2 < i \leq n/2, 0 \leq j \leq n-1}$), $h_H$ is the activation function and $z_0$ is the initial or start state. 
Therefore, starting from the start state ($z_0$) at $t=0$, the model recursively evolves as $t$ increases, according to Equation \ref{eq:second_order}. Thus, when the TRNN reaches a fixed point, i.e., a halting state or dead state, the network dynamics are represented as:
\begin{align}\label{eq:second_order_halt}
    z_{i}^{t} = h_H(W_{ijk}^{t}z_jx_k + b_i)
\end{align}
where $z_{i}^t$ stays in the same state. Now we will prove that a TRNN with two steps (or cycles) is Turing complete and only requires $m+2n+1$ unbounded precision neurons. 
\begin{theorem}\label{thm:2n}
    Given a Turing Machine $M$, with $m$ symbols and $n$ states, there exists a $k$-neuron unbounded precision TRNN with locally connected second-order connections, where $K = m + 2n + 1 $ can simulate any TM in O($T^2$) turns or cycles.
\end{theorem}
\begin{proof}
Recall that the transition rules of a TM $M$ are defined using $\delta$ as follows:
\begin{align}
\delta = \{<r_i, z_j| r_u, z_v A> | i, u = 0,1,2..m, j,v = 1,2..n\}
\end{align}
where we have introduced an additional symbol $s_0$, which represents a blank tape symbol.
It follows that the neural or state transition of a TRNN can be represented as:
\begin{align*}
  Z_{i,j}^{t+1} = h_H(\sum_{j_1,j_2 = 0}^{m+2n} (W^t_{1_{j,j_1,j_2}}Z^{t}_{i-1,j} Z_{i,j_2}^{t} + W^t_{2_{j,j_1,j_2}} Z^{t}_{i,j_1} Z_{i+1,j_2}^{t}) + \theta)   
\end{align*}
where $h_{H_{i,j}}$ is the sigmoid activation function and the neural activities are close to either $0$ or $1$. $F$ 
is another activation function --- the extended version of $h_H$ --- such that:
\begin{multline*}
    F = 2 \delta_{j_0}(1 - \sum_{j_1 = 0}^{m+2n} Z_{i,j_1}^t + \epsilon_1)(Z_{i-1,j_0}^{t} + Z_{i+1,j_0}^{t}) + \delta_{j_0}(1-\sum_{j_1 = 0}^{m+2n} Z_{i-1,j_1}^t + \epsilon_2 + 1 \\ + \sum_{j_1 = 0}^{m+2n} Z_{i+1,j1}^t + \epsilon_2) Z_{i,0}^{t} - (1.5 - \epsilon_4)
\end{multline*}
where $\epsilon_i$ is a small error (value) and $i \in {1,2..4}$ shifts neuron values to be close to $0$ or $1$. Note that the synaptic weights of a TRNN are locally connected and each neuron at the $i_{th}$ position or index is only accessed or connected to the neurons at its immediate left $(i-1)^{th}$ or right $(i+1)^{th}$. If we map these values to a 2D plane, we can easily see that the weight values are uniformly distributed across $i$ (or the $i-axis$), which is independent of the column position of $i$. If we apply Theorem \ref{thm:2n} recursively until the final or halt state, we prove that the TRNN can simulate any TM in O($T^2$) turns or cycles. One could also replace $F$ with the $tanh$ activation function, since $tanh=2(h_H)-1$.
\end{proof}

The mapping for a TRNN can be represented by $T_{w,b}$, where $T_{w,b}^2$ = $T_{M}(z)$. This means that $2$ steps of a TRNN simulate one step of a TM $M$. Let us try to visualize the simulation. First, we will assign the state mapping of Turing machine $M$ to the subset of states $z_t$ of the TRNN. Since $k=m+2n+1$ is the maximum number of neurons with second-order weights, each column of neurons would represent a $K$ dimensional vector with activation function $h_H$, 
such that neuron values are either converging towards $0$ or $1$. 
Ideally, we need $m+2n$ neurons with values equivalent to or close to $0$ and only one neuron with values close to $1$ to represent TM $M$'s states $Q$. We represent the neurons with values close to $1$ as $RN_1$. Therefore, based on the position of $RN_1$, we can have each column representing either the control head or the tape symbol. 

Let us see how $RN_1$ functions in our construction. Assume the position or index $j$ of $RN_1$ is between $0 \leq j \leq m$; then, the current column would correspond to the tape symbol $r_j$. One important thing to remember is that one can imagine TM $M$ in a matrix form by considering the row to correspond to the transition rule and the column to correspond to the tape that is assumed to be infinite. Thus, a blank symbol would be represented when position of $RN_1$ is at $j=0$. On the other hand if position of $RN_1$ is between $m+1 \leq j \leq m+2n$, then the column would correspond to a control head representing the ($(j-m+1)/2$)-th state. As we can see, we have introduced two additional vectors using $RN_1$. In other words, the $RN_1$ position at $j$ now models two values, one at $j=m+2q$ and one at $j=m+2q-1$, that ideally represent the same control head configuration at the $q$-th state extracted from the TM $M$. As discussed before, this construction requires $2$ steps/cycles to represent one step/cycle of a TM $M$. In essence, this simply states that, to simulate a tape's ``left operation'', we will need two steps -- the first is to simulate the left move command of the control head and the second is to create a buffer that stores the ``left operation'' action. We can perform a similar operation to get the ``right operation'' of the tape, thus completing the construction. 

It should be noted that, under this arrangement, we can observe that the product of a blank symbol and the states that lie within the central range are equivalent to $0$, rather than being close to $1$. This is where the special threshold function $F$ comes into play -- it will ensure that the central state lies within the range and does not deviate away from the TM $M$'s configuration.  

Next, we will derive the lower bound of a TRNN for recognizing a small universal Turing machine (UTM). Prior work has proposed four small UTMs \cite{neary2007four} and one such configuration contains $6$ states and $4$ symbols that can simulate any Turing machine in time O($T^6$), where $T$ represents number of steps that the TM requires to compute a final output. As shown in this work, the $UTM_{6,4}$ or $UTM_{7,3}$ is a universal TM, which we state following theorem:
 \begin{theorem}\cite{neary2007four}
\label{thm:product_and_DTM}
 Given a deterministic single tape Turing machine $M$ that runs in time $t$ then any of the above UTMs can simulate the computation of $M$ using space $O(n)$ and time $O(t^2)$.
\end{theorem}
We now use Theorem \ref{thm:product_and_DTM} 
to simulate a UTM, thus showing the minimal number of neurons required by a TRNN to recognize any TM grammar and prove the following corollary:
\begin{corollary}
There exists a $17$ neuron unbounded-precision TRNN that can simulate any TM in O($T^4$), where $T$ is the total number of turns or steps required by the TM to compute the final output.
\end{corollary}
Now we show a much stronger bound -- the TRNN can simulate any TM in real-time (O($T$)) and the UTM in O($T^2$). Note that the previously derived bound for the RNN without memory is O($T^8$); with memory, it is O($T^6$). 
\begin{theorem}
    Given a Turing Machine $M$, with $m$ symbols and $n$ states, there exists a $k$-neuron unbounded precision TRNN with two locally connected second-order connections, where $K = m + n + 1 $ can simulate any TM in real-time or O($T^1$) cycles.
\end{theorem}

\begin{sproof}
We add additional weights ($W^a$) that correspond to a blank tape and $m$ symbols of the Turing Machine and another set of weights ($W$) that keep track of the empty or blank state (this keeps track of states that the control head is not referring to at any given time-step $t$) as well as controller states. The addition of an extra set of second order weight matrices makes this model different compared to the one in Theorem \ref{thm:2n}. Now the weight updates for the TRNN are computed as follows:
\begin{align}
    Z_{i,j}^{t+1} = h_H(\sum_{j_1= 0}^{m} \sum_{j_2 =0}^{n} (W^t_{j,j_1,j_2} Z^{t}_{i+1,j_1} A_{i+1,j_2}^{t} +  \theta)
\end{align}
\begin{align}
    A_{i,j}^{t+1} = h_H(\sum_{j_1= 0}^{m} \sum_{j_2 =0}^{n} (W^at_{j,j_1,j_2} Z^{t}_{i,j_1} A_{i,j_2}^{t} +  \theta^a)
\end{align},
where $h_H$ is the sigmoid activation function while $\theta$ and $\theta^a$ are the extended functions. The rest of the construction follows Theorem \ref{thm:2n} where the left and right tape operations are done in real-time.
\end{sproof}

\begin{corollary}
There exists an $11$ neuron unbounded-precision TRNN that can simulate any TM in O($T^2$), where $T$ is the total number of turns or steps required by the TM in order to compute the final output.
\end{corollary} 
It is worth noting that prior results have shown that RNNs are Turing complete with $40$ unbounded precision neurons. Notably, they operate in O($T^3$) cycles and, with a UTM, it operates in O($T^6$). Our proof shows that the smallest RNN with only \textbf{$11$ unbounded precision} neurons is Turing complete and can work in real-time \textbf{O($T$)}. More precisely, we can construct a UTM with $6$ states and $4$ symbols that can simulate any TM in O($T^2$) cycles. 
In the next section, we will show, with bounded precision, what is the computational limit of the TRNN.



\section{Tensor RNNs with Bounded Precision}
\label{sec:trnn_bounded_prec}

Here, we will revisit some of the tighter bounds placed on TRNNs, based on notions from \cite{omlin1996constructing}, and show that TRNNs are strictly more powerful compared to modern-day RNNs such as LSTMs and GRUs. To do this, we will look at a few definitions with respect to the fixed point and the sigmoid activation function. We will show a vectorized dynamic version of the DFA, which closely resembles the workings of neural networks. Thus, such a DFA can be represented as follows:
\begin{definition}\label{FSA_def_3}\emph{(DFA, vectorized dynamic version)}
  At any time $t$, the state of a deterministic finite-state automaton is represented by the state vector
  $$\bar{Q}^t\in\{0,1\}^n,$$
  whose components $\bar{Q}^t{}_i$ are
  \begin{equation*}
    \bar{Q}^t{}_i = \left\{
      \begin{array}{ll}
        1 & q_i = q^t \\
        0 & q_i \neq q^t.
      \end{array}
    \right.
  \end{equation*}

  \noindent 
  The next state vector $\bar{Q}^{t+1}$ is given by the dynamic relation
  $$\bar{Q}^{t+1} = W^{t+1}\cdot\bar{Q}^t,$$
  where the transition matrix $W^{t+1}$, determined by the transition tensor $W$ and the input vector $I^{t+1}$, is
  $$W^{t+1}=W\cdot I^{t+1}.$$ 
The input vector $I^{t+1}\in\{0,1\}^m$ has the components: 
  \begin{equation*}
    I^{t+1}{}_j = \left\{
      \begin{array}{ll}
        1 & \chara_j = \chara^{t+1} \\
        0 & \chara_j \neq \chara^{t+1},
      \end{array}
    \right.
  \end{equation*}
  and the transition tensor $W$ has the components: 
  \begin{equation*}
    W_i{}^{jk} = \left\{
      \begin{array}{ll}
        1 & q_i = \delta(q_j, \chara_k) \\
        0 & q_i \neq \delta(q_j, \chara_k). 
      \end{array}
    \right.
  \end{equation*}
  In component or neural network form, the dynamic relation explaining next state transition can then be rewritten as: 
  \begin{equation}\label{fsa_dynamics}
    \bar{Q}^{t+1}{}_i = \sum_{jk} W_{ijk}I^{t+1}{}_k\bar{Q}^t{}_j.
  \end{equation}
\end{definition}
\begin{definition}
Let us assume f : Z $\rightarrow Z$ to be a mapping in metric space. Thus, a point $z_f$ $\in$ Z is called a fixed point of the mapping such that f($z_f$) = $z_f$.
\end{definition}
We are interested in showing that the TRNN has a fixed point mapping that ensures that the construction stays stable. To achieve this, we define the stability of $z_f$ as follows:
\begin{definition}
A fixed point $z_f$ is considered stable if there exists a range $R =[i,j] \in Z$ such that $z_f$ $\in$ R and if iterations for $f$ start converging towards $z_f$ for any starting point {$z_s \in $ R}.
\end{definition}
It is shown that a continuous mapping function $f : Z \rightarrow Z$ at least has one fixed point. These definitions are important for showing that a TRNN with the sigmoid activation converges to fixed point and stays stable. 
Next, using fixed point analysis and Lemma \ref{lemma:stack_op_lem_2}, we will prove that a TRNN can simulate any DFA with bounded precision and weights. In particular, we will formally establish:
\begin{theorem}
    Given a DFA $M$ with $n$ states and $m$ input symbols, there exists a $k$-neuron bounded precision TRNN with sigmoid activation function ($h_H$), where $k = n+1$, initialized from an arbitrary distribution, that can simulate any DFA in real-time O($T$).
\end{theorem}
\begin{proof}
It can be seen that, at time $t$, the TRNN takes in the entire transition table as input. This means that a neural network simulator ($Z$) consists of the current state ($Q$) at time $t$ extracted from DFA $M$, the input symbol, and the current operation. The network dynamically evolves to determine next state $Z_{t+1}$ at $t+1$, which ideally represents the next state $Q_{t+1}$ in DFA. As seen by Definition \ref{FSA_def_3}, In particular, when we look at equation \ref{fsa_dynamics}, which represents the state of the DFA, and Equation \ref{eq:second_order_halt}, which represents the state of TRNN, we clearly see that both follow from each other. Thus, by induction, we observe that the TRNN is equivalent to the DFA using a differentiable activation function, since, for large values of $h_H$, the activation function will reach a stable point and converge to either $0$ or $1$, based on the network dynamics. 
\end{proof}

Similar results have been shown in prior work \cite{omlin1996constructing}. Specifically, this effort demonstrated that, using the sigmoid activation function $(h_H)$, one can obtain two stable fixed points in the range [$0-1$], such that state machine construction using a TRNN stays stable until the weight values are within a fixed point. 

It is important to note that prior saturated function analysis results \cite{Merrill19,weiss18} do not provide any bounds and the results are only valid for binary weights. On the other hand, we have derived the true computational power of (tensor) RNNs and our results are much more robust, showing the number of bounded precision neurons required to recognize any DFA in real-time. To the best of our knowledge, first order RNNs require at least \textbf{$2mn - m + 3$} neurons in order to recognize any DFA, where $m$ represents the total number of input symbols and $n$ represents the states of the DFA. In contrast, the TRNN requires only \textbf{($n+1$)} neurons to recognize any DFA.  This is true even with weight values that are initialized using a Gaussian distribution. Finally, our construction shows that the TRNN can encode a transition table into its weights, thus insertion and extraction of rules is much more stable with the TRNN. As a result, the model is also desirably \textbf{interpretable} by design and extraction. 

\section{Experimental Setup and Result}
\label{sec:experiments}

We conduct experiments on the $7$ Tomita grammars that fall under regular grammars and are considered to be the standard benchmark for evaluating memory-less models. We randomly sample strings of lengths up to $50$ in order to create train and validation splits. Each split contains $2000$ samples obtained by randomly sampling without replacement from the entire distribution. This is done so as to ensure that both distributions are far apart from each other. Second, we create two test splits with $1000$ samples in each partition. Test set \#$1$ contains samples different from prior splits and has strings up to length $60$  whereas test set \#$2$ contains samples of length up to $120$. We experiment with two widely popular architectures to serve as our baseline models, the long short-term memory (LSTM) network and the neural transformer. Mainly, we perform a grid search over the learning rate, number of layers, number of hidden units, and number of heads in order to obtain optimal (hyper-)parameter settings. 
Table \ref{tab:Tomita_settings} provides the range and parameter settings for each model. 
\begin{table}
\centering
\begin{tabular}{|c|c|}
    \hline
    \# & Definition \\
    \hline
    1 & $a^*$ \\
    2 & $(ab)^*$ \\
    3 & Odd ${\llbracket .... \rrbracket}_a$ of $a$'s must be followed by even ${\llbracket .... \rrbracket}_b$ of $b$'s \\
    4 & All strings without the trigram $aaa$ \\
    5 & Strings $I$ where ${\llbracket .... \rrbracket}_a(I)$ and ${\llbracket .... \rrbracket}_b(I)$ are even \\
    6 & Strings $I$ where ${\llbracket .... \rrbracket}_a(I)$ $\equiv_3$$ \llbracket .... \rrbracket_b(I)$ \\
    7 & $b^*a^*b^*a^*$ \\
    \hline
\end{tabular}%
\caption{Definitions of the Tomita languages. Let ${\llbracket .... \rrbracket}_{\sigma}(I)$ denote the number of occurrences of symbol $\sigma$ in string $I$. Let $\equiv_3$ denote equivalence mod $3$.}
\label{tab:tomita}
\end{table}
To better demonstrate the benefits of the TRNN and to show that the empirical model closely follows our theoretical results, we only experimented with a single-layer model with a maximum of $32$ hidden units. Weights were initialized using Xavier initialization for the baseline models, whereas for the TRNN, weights were initialized from a centered Gaussian distribution. In Table  \ref{tab:Tomita_performance}, we observe that the transformer-based architecture struggles in recognizing longer strings, despite achieving $100$\% accuracy on the validation splits. As has been shown in prior efforts, earlier transformer-based architectures heavily depend on positional encodings, which work well for machine translation. Nevertheless, for grammatical inference, one needs to go beyond the input distribution in order to efficiently recognize longer strings. Since positional encodings do not contain information on longer strings (or on those longer than what is available in the training set), a transformer model is simply not able to learn the state machine and thus struggles to recognize the input grammar. In addition, transformers need more parameters in order to recognize the input distribution; this is evident from Table \ref{tab:Tomita_settings}, where we show that transformers even with a large number of parameters struggle to learn complex grammars such as Tomita $3,5$ and $6$. 

However, as we can see from Table \ref{tab:Tomita_performance}, the LSTM matches TRNN performance but requires more parameters (as is also evident in Table \ref{tab:Tomita_settings}, which reports best settings for the LSTM). Third, when tested on longer strings, it becomes evident that the TRNN better understands the rules that can be inferred from the data, as opposed to simply memorizing the patterns (and thus only overfitting). This is clearly seen in Table \ref{tab:Tomita_performance_2}, where the LSTM performance starts dropping (whereas, in contrast, minimal loss is observed for the tensor model). Furthermore, in the next section, we present a stability analysis showing that one can stably extract automata from TRNN. In Figure \ref{fig:automata_extr}, we show the comparison between ground-truth (or oracle) and extracted state machines. 

These simulation results support our theoretical findings that a first-order neural model, e.g., LSTM, transformer, requires more parameters to learn a DFA than the TRNN. In other words, for a DFA with $100$ states ($m$) and $70$ input symbols ($n$), theoretically, a tensor RNN would only require $101$ neurons in order to successfully recognize the grammar. First-order RNNs, on the other hand, such as the LSTM, would require $2mn - m + 3n+1$ or $14111$ neurons. This is an order of magnitude larger than that required by a TRNN and such a trend is, as we have shown, empirically observable. 
\begin{table}
    \centering
    \begin{tabular}{l|r|r||r|r||r|r}
    \multicolumn{1}{l}{}&\multicolumn{2}{|c}{\begin{tabular}[x]{@{}c@{}}\textbf{$Attn$}\\\end{tabular}}&\multicolumn{2}{|c}{\begin{tabular}[x]{@{}c@{}}\textbf{$LSTM$}\\\end{tabular}}&\multicolumn{2}{|c}{\begin{tabular}[x]{@{}c@{}}\textbf{$\textbf{TRNN(ours)}$}\\\end{tabular}}\tabularnewline
    \multicolumn{1}{l}{}&\multicolumn{1}{|c}{\begin{tabular}[x]{@{}c@{}}\textbf{n=60}\\\end{tabular}}&\multicolumn{1}{|c}{\begin{tabular}[x]{@{}c@{}}\textbf{120}\\\end{tabular}} &\multicolumn{1}{|c}{\begin{tabular}[x]{@{}c@{}}\textbf{60}\\\end{tabular}}&\multicolumn{1}{|c}{\begin{tabular}[x]{@{}c@{}}\textbf{120}\\\end{tabular}}&\multicolumn{1}{|c}{\begin{tabular}[x]{@{}c@{}}\textbf{60}\\\end{tabular}}&\multicolumn{1}{|c}{\begin{tabular}[x]{@{}c@{}}\textbf{120}\\\end{tabular}}\tabularnewline
\hline
        \textit{Tm-1} & $100$ & $100$ & $100$ & $100$ &$100$ &$100$ \tabularnewline
        \textit{Tm-2}& $100$ & $100$ & $100$ & $100$ &$100$ &$100$ \tabularnewline
        \textit{Tm-3} & $75.00$ & $9.80$ & $99.99$ & $98.99$ &$100$ &$100$   \tabularnewline
        \textit{Tm-4} & $100$ & $92.00$ & $100$ & $100$ &$100$ &$100$
        \tabularnewline
        \textit{Tm-5} & $29.80$ & $0.0$ & $99.99$ & $98.50$ &$100$ &$100$  \tabularnewline
        \textit{Tm-6} & $88.00$ & $0.0$ & $100$ & $99.99$ &$100$ &$99.99$  \tabularnewline
        \textit{Tm-7} & $99.5$ & $99.20$ & $100$ & $100$ &$100$ &$100$  \tabularnewline
        
    \end{tabular}
    %
\caption{Percentage of correctly classified strings of ANNs trained on Tomita (Tm) languages (over $5$ trials). We report the mean accuracy (per model). The test set contained a mix of short and long samples of length up to $n=60 $ and $n = 120$ with both positive and negative samples, respectively.}
\label{tab:Tomita_performance}
\end{table}

\begin{table}

    \centering
    \begin{tabular}{l|r|r||r|r||r|r}
    \multicolumn{1}{l}{}&\multicolumn{2}{|c}{\begin{tabular}[x]{@{}c@{}}\textbf{$Attn$}\\\end{tabular}}&\multicolumn{2}{|c}{\begin{tabular}[x]{@{}c@{}}\textbf{$LSTM$}\\\end{tabular}}&\multicolumn{2}{|c}{\begin{tabular}[x]{@{}c@{}}\textbf{$\textbf{TRNN(ours)}$}\\\end{tabular}}\tabularnewline
    \multicolumn{1}{l}{}&\multicolumn{1}{|c}{\begin{tabular}[x]{@{}c@{}}\textbf{n=200}\\\end{tabular}}&\multicolumn{1}{|c}{\begin{tabular}[x]{@{}c@{}}\textbf{400}\\\end{tabular}} &\multicolumn{1}{|c}{\begin{tabular}[x]{@{}c@{}}\textbf{200}\\\end{tabular}}&\multicolumn{1}{|c}{\begin{tabular}[x]{@{}c@{}}\textbf{400}\\\end{tabular}}&\multicolumn{1}{|c}{\begin{tabular}[x]{@{}c@{}}\textbf{200}\\\end{tabular}}&\multicolumn{1}{|c}{\begin{tabular}[x]{@{}c@{}}\textbf{400}\\\end{tabular}}\tabularnewline
\hline
        \textit{Tm-1} & $72$ & $50$ & $80$ & $72.50$ &$100$ &$100$ \tabularnewline
        \textit{Tm-2}& $60$ & $35$ & $85.50$ & $75.50$ &$100$ &$100$ \tabularnewline
        \textit{Tm-3} & $1.50$ & $0.0$ & $80.00$ & $60.00$ &$99.50$ &$99.00$   \tabularnewline
        \textit{Tm-4} & $45$ & $15.0$ & $89.50$ & $80$ &$100$ &$100$
        \tabularnewline
        \textit{Tm-5} & $0.0$ & $0.0$ & $82.50$ & $68.50$ &$99.99$ &$98.50$  \tabularnewline
        \textit{Tm-6} & $0.0$ & $0.0$ & $78.50$ & $63.50$ &$99.99$ &$99.99$  \tabularnewline
        \textit{Tm-7} & $42.50$ & $8.50$ & $79.00$ & $61.00$ &$100$ &$99.99$  \tabularnewline
        
    \end{tabular} %
 \caption{We report the mean accuracy for each model. The test set contained a mix of very long samples of length up to $n=200 $ and $n = 400$ with both positive and negative samples, respectively.}
 \label{tab:Tomita_performance_2}
\end{table}

\begin{table}

    \centering
    \begin{tabular}{l|r|r||r|r||r|r}
    \multicolumn{1}{l}{}&\multicolumn{2}{|c}{\begin{tabular}[x]{@{}c@{}}\textbf{$Attn$}\\\end{tabular}}&\multicolumn{2}{|c}{\begin{tabular}[x]{@{}c@{}}\textbf{$LSTM$}\\\end{tabular}}&\multicolumn{2}{|c}{\begin{tabular}[x]{@{}c@{}}\textbf{$\textbf{TRNN(ours)}$}\\\end{tabular}}\tabularnewline
    \multicolumn{1}{l}{\textbf{ }}&\multicolumn{1}{|c}{\begin{tabular}[x]{@{}c@{}}\textbf{V-Acc}\\\end{tabular}}&\multicolumn{1}{|c}{\begin{tabular}[x]{@{}c@{}}\textbf{Itr}\\\end{tabular}} &\multicolumn{1}{|c}{\begin{tabular}[x]{@{}c@{}}\textbf{V-Acc}\\\end{tabular}}&\multicolumn{1}{|c}{\begin{tabular}[x]{@{}c@{}}\textbf{Itr}\\\end{tabular}}&\multicolumn{1}{|c}{\begin{tabular}[x]{@{}c@{}}\textbf{V-Acc}\\\end{tabular}}&\multicolumn{1}{|c}{\begin{tabular}[x]{@{}c@{}}\textbf{Itr}\\\end{tabular}}\tabularnewline
\hline
        \textit{Tm-1} & $100$ & $12$ & $100$ & $5$ &$100$ &$12$ \tabularnewline
        \textit{Tm-2}& $100$ & $11$ & $100$ & $4$ &$100$ &$7$ \tabularnewline
        \textit{Tm-3} & $98.99$ & $18$ & $100$ & $12$ &$100$ &$18$   \tabularnewline
        \textit{Tm-4} & $100$ & $9$ & $100$ & $9$ &$100$ &$16$
        \tabularnewline
        \textit{Tm-5} & $99.99$ & $13$ & $100$ & $9$ &$100$ &$15$  \tabularnewline
        \textit{Tm-6} & $100$ & $8$ & $100$ & $12$ &$100$ &$18$  \tabularnewline
        \textit{Tm-7} & $100$ & $13$ & $100$ & $21$ &$100$ &$29$  \tabularnewline
        
    \end{tabular}
\caption{Percentage of correctly classified strings of ANNs trained on the $Tomita$ languages (across $5$ trials). We report the mean accuracy for each model on the validation test including mean number of epochs required by the model in order to achieve perfect validation accuracy.}
\label{tab:Tomita_train_perf}
\end{table}


\section{Automata Extraction}
\label{sec:extraction}

Next, we test how well the TRNN model performs when we stably extract automata from the trained network stably. There are several automata extraction approaches published in the literature, such as methods based on clustering \cite{omlin1996extraction}, state-merging \cite{wang2019stateregularized}, and even querying algorithms such as L-star \cite{weiss2017extracting}. In Figure \ref{fig:automata_extr}, we compare the ground truth DFA against extracted DFA; the majority of the time ($8/10$), we observed that the TRNN extracts stable DFA, whereas the LSTM reduces to $6/10$ and vanilla RNN $5/10$. In our experiments, we use a clustering-based approach \cite{omlin1996extraction} to extract automata. Note that we notably observed slightly better performance when using a self-organizing map, as opposed to simple K-means. 

We observe that the vanilla RNN and LSTM benefit from training for a longer duration in terms of extraction. In Table \ref{tab:Tomita_train_perf}, we observe that all models reach almost $100\%$ accuracy on validation splits within $15$ epochs, however training for longer always helps with stable automata extraction. For instance, if we use early stopping and stop the training  whenever the model reaches $100$\% accuracy on the validation set for $5$ consecutive epochs, we observe that the automata extracted from the model are unstable. 

Another key observation is that the DFA extracted from such models have a large number of states and, as a result, minimizing them takes a huge amount of time. We set a $1500$ second threshold/cut-off time where, if a model is unable to extract DFA within this given timeframe, we consider it as unstable and report this finding. For instance, for the LSTM, the success rate drops to $4/10$ while for the RNN, it is $3/10$. On other hand, the TRNN can still extract the DFA but with exponentially more states and minimizing it necessary does not produce the minimal DFA. Thus, training for longer durations always helps, leading to more stable extraction. 

\begin{figure}[!t]
    \centering
    \includegraphics[width=1.0\columnwidth]{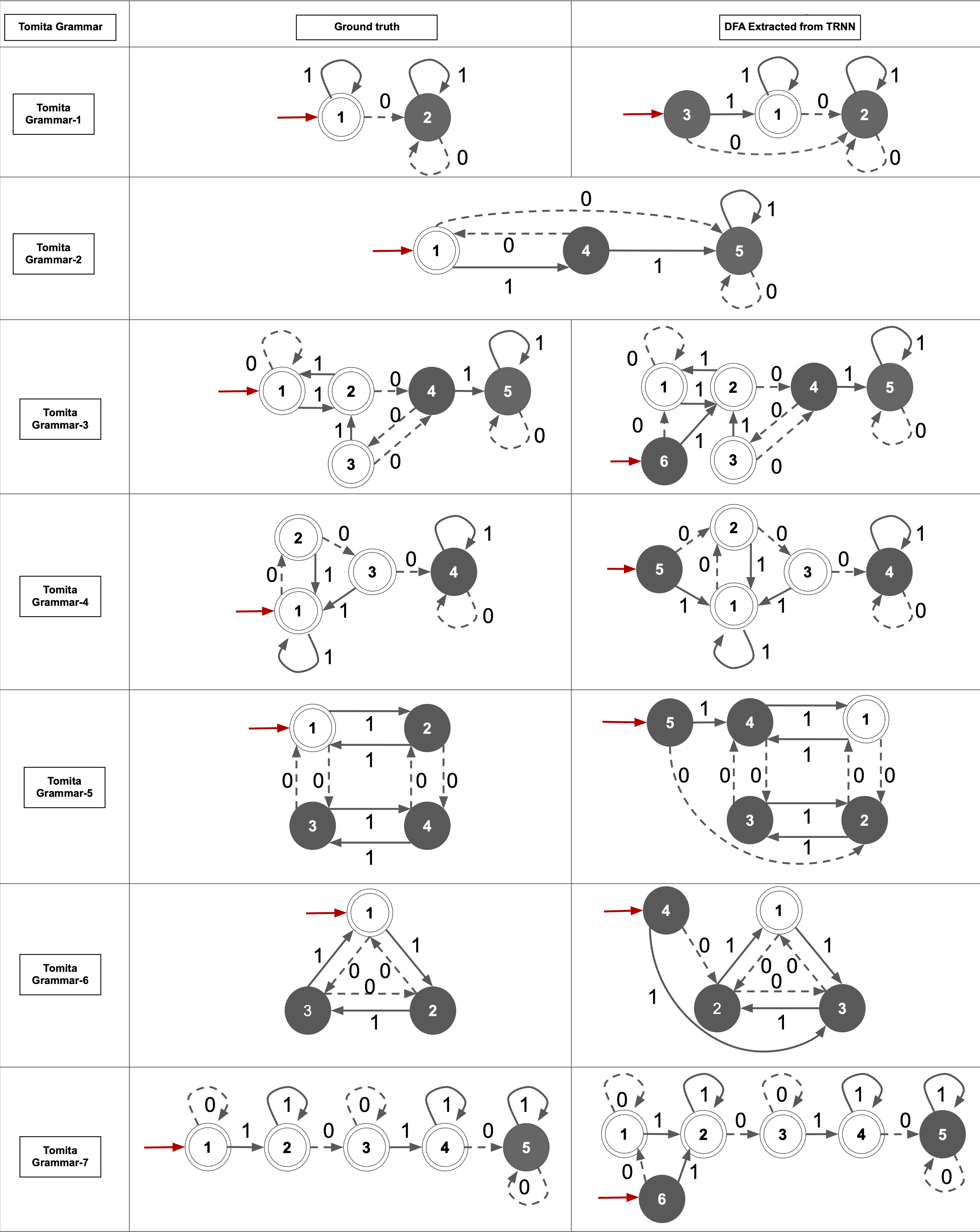}
    \caption{Comparison of the ground-truth (or oracle) DFA and the minimal DFA extracted from the $2^{nd}$ order recurrent neural network (TRNN) for all $7$ Tomita Grammars}
    \label{fig:automata_extr}
\end{figure}

\section{Conclusion}
\label{sec:conclusion}

This work presents a theoretical breakthrough demonstrating that a (tensor) recurrent neural network (RNN) with only $11$ unbounded precision neurons is Turing-complete, and importantly, operates in real-time. This is the first memory-less model that works in real-time and requires such a minimal number of neurons to achieve equivalence. In addition, we show that an $11$-neuron RNN with second-order synaptic connections can simulate any Turing machine in just $2$ steps, while previous models with $40$ neurons require $6$ steps to achieve the same result. We also prove that, with bounded precision, a tensor RNN (TRNN) can simulate any deterministic finite automaton (DFA) using only $n+1$ neurons, where $n$ is the total number of states of the DFA. We show that our construction is more robust than previous saturated function analysis approaches and, notably, provides an upper bound on the number of neurons and weights needed to achieve equivalence. Finally, we evaluate our TRNN model on the challenging Tomita grammars, demonstrating its superior performance compared to transformers and LSTMs, even when using fewer parameters. Our results highlight the potential of TRNNs for stable automata extraction processes and makr a promising direction for future responsible and interpretable AI research.





\newpage
\bibliographystyle{acm}
\bibliography{custom}

\newpage
\appendix
\section{Appendix A: Hyperparameter Settings}
\label{sec:appendix}

\begin{table*}[!t]
\caption{We report different hyperparameter ranges for the grid search conducted in this work. We also show the best values considered for each model. Note that ``N/A'' represents not applicable. 
We use positional encodings for all transformer based models. }
\label{tab:Tomita_settings}
    \centering
     \resizebox{\textwidth}{!}{%
    \begin{tabular}{l|r|r||r|r||r|r}
    \multicolumn{1}{l}{}&\multicolumn{2}{|c}{\begin{tabular}[x]{@{}c@{}}\textbf{$Transformer$}\\\end{tabular}}&\multicolumn{2}{|c}{\begin{tabular}[x]{@{}c@{}}\textbf{$LSTM$}\\\end{tabular}}&\multicolumn{2}{|c}{\begin{tabular}[x]{@{}c@{}}\textbf{$\textbf{TRNN(ours)}$}\\\end{tabular}}\tabularnewline
    \multicolumn{1}{l}{\textbf{Hyper-parameters}}&\multicolumn{1}{|c}{\begin{tabular}[x]{@{}c@{}}\textbf{Range}\\\end{tabular}}&\multicolumn{1}{|c}{\begin{tabular}[x]{@{}c@{}}\textbf{Best}\\\end{tabular}} &\multicolumn{1}{|c}{\begin{tabular}[x]{@{}c@{}}\textbf{Range}\\\end{tabular}}&\multicolumn{1}{|c}{\begin{tabular}[x]{@{}c@{}}\textbf{Best}\\\end{tabular}}&\multicolumn{1}{|c}{\begin{tabular}[x]{@{}c@{}}\textbf{Range}\\\end{tabular}}&\multicolumn{1}{|c}{\begin{tabular}[x]{@{}c@{}}\textbf{Best}\\\end{tabular}}\tabularnewline
\hline
        \textit{Hidden-size} & $[32-256]$ & $192$ & $[32-128]$ & $128$ &$[4-32]$ &$16$ \tabularnewline
        \textit{Attention Heads}& $[1-4]$ & $3$ & $N/A$ & $N/A$ &$N/A$ &$N/A$ \tabularnewline
        \textit{number of layers} & $[1-5]$ & $3$ & $[1-3]$ & $2$ &$1$ &$1$   \tabularnewline
        \textit{learning rate} & $[1e-4, 3e-5]$ & $2e-4$ & $[1e-3, 2e-4]$ & $1e-4$ &$[1e-2, 2e-4]$ &$1e-3$
        \tabularnewline
        \textit{Optimizer} & $[SGD, Adam, RMSprop]$ & $Adam$ & $[SGD, Adam, RMSprop]$ & $SGD$ &$[SGD, Adam, RMSprop]$ &$SGD$  \tabularnewline
        
    \end{tabular}%
    }
\end{table*}

This section reports hyper-parameter for all 3 models, specifically in Table \ref{tab:Tomita_settings}. Observe that the TRNN requires a smaller number of neurons to obtain stable results on longer strings and desirably results in stable automata extraction. 
The codebase is written in TensorFlow 2.1, and all experiments are performed on an Intel Xeon server with a maximum clock speed of $3.5$GHz with four $2080$Ti GPUs with $192$GB RAM.

\section{Appendix B: Limitations}

This work focuses on grammatical inference and does not provide formal performance guarantees for natural language processing benchmarks. Note that the TRNN uses tensor connections which are computationally expensive and pose two main challenges. First, training tensor synaptic connections are difficult using backpropagation of errors through time, as the underlying computational graph for a tensor-based model is very dense, leading to a challenging credit assignment problem. Second, scalability is a concern since larger tensors of connections require vastly greater amounts of compute power. Even vanilla RNNs struggle to scale to billions of parameters due to their sequential dependencies. However, we note that the TRNN's interpretable nature provides flexibility in inserting, modifying, and extracting rules, which could make it more trustworthy than first-order ANNs. Therefore, exploring approximation approaches to tensor-based computing that can provide a computational benefit without compromising the TRNN's generalization ability would be worthwhile.

\section{Appendix C: Ethics and Societal Impact}
The Tensor Recurrent Neural Network (TRNN) proposed in this work can encode rules and generate more interpretable outcomes than other first-order ANNs. As a result, TRNNs could significantly impact society, particularly in domains such as healthcare and law, where transparency and trustworthiness are paramount. However, It is worth noting that this model does not guarantee ethical or unbiased outcomes, as it is still susceptible to biases in the training data. Therefore, it is essential to carefully evaluate the rules generated by such models and ensure that they align with the ethical framework developed by the underlying agency or governing body. By doing so, we can ensure that the benefits of TRNNs are maximized while minimizing the potential risks and negative impacts they may have. This is a critical consideration as we continue to develop and deploy more advanced AI systems in various domains. 
\end{document}